\documentclass{article}

\usepackage{arxiv}

\usepackage[utf8]{inputenc} % allow utf-8 input
\usepackage[T1]{fontenc}    % use 8-bit T1 fonts
\usepackage{hyperref}       % hyperlinks
\usepackage{url}            % simple URL typesetting
\usepackage{booktabs}       % professional-quality tables
\usepackage{amsfonts}       % blackboard math symbols
\usepackage{nicefrac}       % compact symbols for 1/2, etc.
\usepackage{microtype}      % microtypography

\usepackage{lipsum}         % Can be removed after putting your text content
\usepackage{graphicx}
\usepackage{natbib}
\usepackage[most]{tcolorbox} % 核心宏包
\usepackage{amsmath}         % 处理数学公式
\usepackage{amssymb}
\usepackage{amsthm}
\usepackage{doi}
\usepackage{enumitem}
\usepackage{adjustbox}
\usepackage{array}
\usepackage{multirow}
\usepackage{diagbox}
\usepackage{makecell}
\usepackage{subcaption}
\usepackage{authblk}
\usepackage{wrapfig}
\usepackage{cleveref}
\usepackage{float}
\usepackage[section]{placeins}
\usepackage{algorithm}
\usepackage{algorithmic}
\newtheorem{theorem}{Theorem}
\newtheorem{corollary}{Corollary}
\newtheorem{lemma}{Lemma}
\theoremstyle{remark}
\newtheorem{remark}{Remark}

\graphicspath{{images/}}

\title{DEFT: Data-Efficient Frequency-domain Top-k Sampling via Inverse Discrete Fourier Transform for Spatiotemporal Dynamical Systems Modeling}

\newif\ifuniqueAffiliation
\uniqueAffiliationtrue

\author{
    \fontsize{12pt}{14pt}\selectfont
    \textbf{Hengbo Xiao\textsuperscript{1*}}, \textbf{Jiale Liu\textsuperscript{1*}}, 
    \textbf{Jiahao Song\textsuperscript{1}},
    \textbf{Guannan He\textsuperscript{1\textdagger}}
    \vspace{6pt} \\
    \textsuperscript{1} Peking University \quad 
    \texttt{gnhe@pku.edu.cn}
}

\hypersetup{
}

\begin{document}
\maketitle
\renewcommand\thefootnote{}
\makeatletter
\def\Hy@footnote@currentHref{author.footnote}% non-empty anchor for hyperref
\makeatother
\footnotetext{%
\textsuperscript{*} Equal contribution \quad
\textsuperscript{\textdagger} Corresponding author
}% avoid hyperref empty anchor warning
\renewcommand\thefootnote{\arabic{footnote}}
\addtocounter{footnote}{-1}

\begin{abstract}
Modeling spatiotemporal dynamical systems governed by partial differential equations (PDEs) poses two major challenges: it either requires expensive physics-based simulators that entail iterative numerical solving at high computational cost, or it depends on abundant training data, yet purely data-driven models often generalize poorly to downstream dynamic operating conditions. We propose DEFT, a frequency-domain data sampling method that identifies the dominant Fourier modes of a physical system and systematically varies the corresponding amplitudes and phases to generate physically consistent training data via the inverse discrete Fourier transform. In addition, we derive a generalization bound of this method. We note that it also provides a theoretically principled criterion for selecting $K$. We evaluate the proposed method through three sets of experiments, each targeting a distinct aspect of its utility. First, we validate the framework on canonical PDEs solving demonstrating that it outperforms traditional methods when the system is dominated by a few prominent frequency components. Second, we employ DEFT as a data-value filter on the diffusion--sorption and Burgers equations of PDEBench, showing that it reduces data requirements by $40\%$ while sacrificing less than $2\%$ in predictive accuracy. Third, to evaluate DEFT for more challenging and practically relevant problems, we validate it in the battery degradation PDE system, achieving consistently high predictive accuracy across various test datasets with $R^2$ values exceeding $0.99$. Moreover, the learned frequency-domain features transfer to other battery chemistries with only $20\%$ of the fine-tuning data. These results demonstrate that DEFT is an effective data-sampling method for efficient operator learning.
\end{abstract}

\section{Introduction}
Spatiotemporal dynamical systems pervade science and engineering. Their state evolution, typically described by partial differential equations (PDEs), exhibits high nonlinearity and multi-scale coupling: in lithium-ion batteries, ion migration spans seconds to hours while capacity fading unfolds over years. Experimental measurements of such systems are costly and high-fidelity numerical simulations are computationally intensive (a single pseudo-two-dimensional (P2D) battery degradation simulation may take hours to days), so acquiring training data that cover the full operating space is a formidable challenge. Existing modeling paradigms face a dilemma: physics-based models are accurate but too expensive for large-scale data generation, whereas data-driven models are efficient at inference but generalize weakly beyond their training distribution. Operator networks such as DeepONet~\citep{lulearning2021} and Fourier Neural Operator~\citep{li2021fourierneuraloperatorparametric} are expressive enough to learn solution operators, yet their generalization ability hinges on the choice of input functions used for training. The question of how to acquire the most representative training data under a finite budget of expensive oracle calls has become a key bottleneck for scientific machine learning.

Time series from dynamical systems have distinctive structure: adjacent states are strongly correlated, evolution is non-stationary, and the spectral energy distribution is structured. Conventional sampling methods clash with these properties. Random sampling assumes i.i.d.\ samples; uniform sampling presumes a stationary input distribution; stratified sampling requires partitions that are difficult to construct in high-dimensional spatiotemporal state spaces; Bayesian optimization incurs $O(n^{3})$ iteration costs, and its uncertainty modeling is fragile for temporally correlated signals. Fundamentally, all of them collect data \emph{passively} in the time domain, without exploiting spectral structure to inform sampling decisions.

A frequency-domain viewpoint offers a way out. Response signals of physical systems are typically \emph{low frequency dominated}: a small number of frequency components carry the vast majority of the energy: battery current profiles, for instance, concentrate in the $0.001$--$0.1$\,Hz band of macroscopic cycling, while components above $1$\,Hz mainly reflect measurement noise and power-electronics ripple. Retaining the energy-dominant components therefore suffices both for high-fidelity signal reconstruction and, crucially, for capturing the dominant dynamics of the underlying system. This sparsity in the frequency domain aligns with compressed sensing theory~\citep{donoho2006compressed}, which guarantees exact reconstruction of sparse signals from a small set of dominant modes.

Building on this observation, we propose an IDFT-based frequency-domain sampling optimization method, organized as a three-tier framework of physics mechanism, frequency-domain generation, and operator learning. At the foundation, a physics-based model provides high-fidelity data; at the intermediate layer, spectral analysis of a small observation set identifies the dominant frequencies with the largest energy contribution, and the IDFT extends these spectral characteristics into diverse synthetic input waveforms, which the physics model then labels into a training set; at the top layer, a deep operator network approximates the input-to-output mapping. Unlike conventional sampling, which selects a subset from an existing large dataset, our method uses small-sample spectral features to guide the physics model in generating a large, diverse training set.

The framework is accompanied by an adaptive IDFT sampling algorithm that ranks frequency components by average spectral energy, truncates the spectrum to energy-dominant components, and synthesizes diverse training waveforms by changing amplitudes and phases in the retained subspace. According to our analysis, the time complexity of this algorithm is $O((M+M_{\mathrm{synth}})N\log N)$ and the space complexity is $O((M+M_{\mathrm{synth}})N)$. In addition, we prove a generalization error bound for IDFT sampling that is decomposed into the deployment and truncation error $O((L_F+L_N)K^{-(\alpha-1/2)})$, the approximation error $O(m^{-s/K})$, the statistical error $O(\sqrt{\log(1/\delta)/n})$, and the optimization residual. Note that $K$ pushes the truncation and approximation terms in opposite directions, so there is an optimal budget $K^{*}$. This offers a criterion for selecting the optimal $K$.

Experiments confirm the framework in three domains. In PDE operator learning (Burgers and Allen--Cahn), DEFT sampling is best whenever the solution is dominated by several frequencies. Furthermore, spectral analysis can also be used in compressing the diffusion--sorption and Burgers PDE operator training set by half at a few precision cost. In P2D battery capacity prediction, 40--60 retained low-frequency components achieve $R^2>0.99$, generalizing to constant-current constant-voltage (CC-CV) charging protocols and electric-vehicle (EV) driving profiles. The model, originally developed for the Li-ion battery, transfers to other battery chemistries and attains $R^2>0.96$ with at most 20\% fine-tuning data.

\section{Related Work}
\paragraph{Data sampling methods} Classical sampling strategies include random, uniform (grid), Latin hypercube sampling (LHS)~\citep{mckay1979comparison}, and Bayesian optimization~\citep{snoek2012practical}. Random sampling and uniform sampling ignore the signal structure and suffer from the curse of dimensionality. LHS enforces uniform marginal coverage, but relies on a linear-projection assumption that breaks for temporally correlated data. Bayesian optimization builds a Gaussian-process surrogate whose cost grows cubically with sample size. Modern active learning methods select maximum uncertain samples but suffer from a circular dependence on the initial model. Conventional generative approaches (GANs, VAEs, diffusion models) require abundant training data and cannot guarantee physical consistency. More importantly, all operate in the time or feature domain and lack generalization guarantees.

\paragraph{Frequency-domain analysis} Frequency domain analysis is widely used in control systems, and the most important part of it is the Fourier transform. Fourier analysis~\citep{oppenheim1997signals} and Fast Fourier Transform (FFT) reduce the computational complexity of the Discrete Fourier Transform (DFT) from $O(N^2)$ to $O(N \log N)$, which is the basis of our work. Fourier modes have been used for the generation of synthetic data in traffic and climate modeling~\citep{xu2021a,wu2024earthfarsser}. Fourier feature networks have also been studied for solving multi-scale PDEs~\citep{wang2021eigenvector}, but in previous work, the choice of bandwidth was heuristic, without a quantitative link between spectral truncation and generalization.

\paragraph{Operator learning} DeepONet~\citep{lulearning2021} and the Fourier Neural Operator~\citep{li2021fourierneuraloperatorparametric,kovachki2023neural} learn mappings between function spaces; physics-informed networks~\citep{raissi2019physics} incorporate PDE residuals as a regularization term in the loss function. These architectures have achieved great success in many areas. However, they can only generalize well when the training data cover the critical regions of the input function space. The lack of generalization performance limits the practical use of these methods.

\paragraph{Generalization theory} The goal of machine learning generalization theory is to understand the difference between training error and test error, and to provide provable guarantees for model performance. Rademacher complexity~\citep{bartlett2002rademacher,mohri2018foundations}, uniform stability~\citep{bousquet2002stability,hardt2016train}, concentration inequalities~\citep{mcdiarmid1989method}, spectrally normalized bounds~\citep{bartlett2017spectrally}, and approximation theory on Sobolev classes~\citep{yarotsky2017error,petersen2018optimal,adams2003sobolev,evans2010partial} provide math tools to analyze it. But classical generalization theory has its limits when it comes to scientific computing scenarios. Their i.i.d. assumption is broken by non-standard sampling strategies like active sampling and frequency domain truncation. The looseness of its bounds makes it difficult to guide practical applications. Its requirement for a finite hypothesis space conflicts with the over-parameterized nature of modern deep learning. These limitations show the need for a new theoretical framework that can handle non-standard sampling, work with infinite-dimensional function spaces, provide tight bounds, and offer practical guidance.

\section{Method: DEFT Sampling}
\label{sec:method}

\subsection{Overview and Problem Formulation}

Let $\mathcal{X}\subset\mathbb{R}^d$ be a space of discretized input signals, $\mathcal{D}$ a deployment distribution on $\mathcal{X}$, and $F:\mathcal{X}\to\mathbb{R}$ a target solution operator. Given an operator-network hypothesis class $\mathcal{H}$ and a loss $\ell$, the goal is to minimize the population risk $\mathcal{E}(N_\theta)=\mathbb{E}_{x\sim\mathcal{D}}\,\ell(N_\theta(x),F(x))$ over $N_\theta\in\mathcal{H}$, using at most $B$ oracle calls for labels, by training on the empirical risk
\begin{equation}
\hat{\mathcal{E}}_{\mathcal{S}}(N_\theta)=\frac{1}{|\mathcal{S}|}\sum_{(x,y)\in\mathcal{S}}\ell(N_\theta(x),y).
\label{eq:erm}
\end{equation}

\begin{figure*}[!htbp]
\centering
\includegraphics[width=0.85\textwidth]{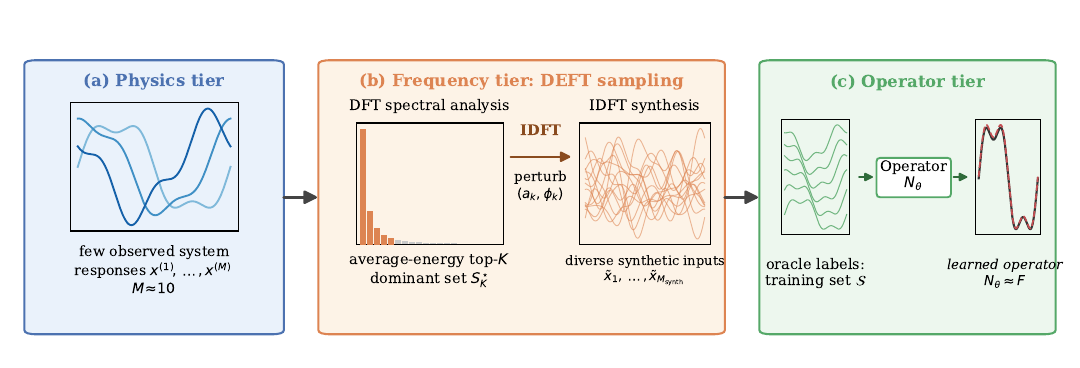}
\caption{The framework. (a) Example observed input fields. (b) Based on DFT analysis, the average-energy top-$K$ frequencies are selected; DEFT perturbs their amplitudes and phases to produce a variety of synthesized input fields. (c) The synthesized fields form a large and physically consistent training set}
\label{fig:framework_general}
\end{figure*}

The classical sampling mentioned above selects a subset from an existing large dataset. In complex-system modeling one instead faces a \emph{small-sample dilemma}: only a handful of observations $\mathcal{D}_{\mathrm{obs}}=\{x^{(1)},\dots,x^{(M)}\}$ (e.g. $M\approx 10$ physics records) are available, far from enough to train a deep model. We therefore replace "subset selection" with \emph{data generation}: we extract the spectral characteristics of the system from $\mathcal{D}_{\mathrm{obs}}$ and synthesize a diverse large training set $\mathcal{S}$ of size $M_{\mathrm{synth}}\gg M$. The pipeline is illustrated in (Figure~\ref{fig:framework_general}).

\subsection{Spectral Analysis and Energy-Dominant Truncation}

For a time series length-$N$ $\{x_n\}$, the DFT and its inverse are
\begin{equation}
X_k=\sum_{n=0}^{N-1}x_n e^{-\mathrm{i}\frac{2\pi}{N}kn},\qquad
x_n=\frac{1}{N}\sum_{k=0}^{N-1}X_k e^{\mathrm{i}\frac{2\pi}{N}kn}.
\label{eq:dft}
\end{equation}
By Parseval's identity, $\|x\|_2^2=\frac{1}{N}\sum_{k}|X_k|^2$: the energy is conserved between domains, so the contribution of every frequency component is explicitly measurable. Given $\mathcal{D}_{\mathrm{obs}}$, we compute the spectrum $X^{(i)}$ for each sample and the \emph{average energy} of the cross-sample.
\begin{equation}
\bar{E}_k=\frac{1}{M}\sum_{i=1}^{M}\bigl|X_k^{(i)}\bigr|^2,
\label{eq:avgenergy}
\end{equation}
and sort the frequencies in descending order. Averaging suppresses per-sample noise and extracts the system's stable dominant frequencies; empirically $M=10$ samples suffice because the dominant frequencies are dictated by the system's intrinsic dynamics rather than by particular inputs.

Given a budget $K_{\mathrm{freq}}\ll N$, we retain the top-$K_{\mathrm{freq}}$ components $\mathcal{K}_{\mathrm{retain}}$ and eliminate the rest. It can be written as $\tilde{x}=\mathcal{F}^{-1}[H\odot X]$, where $H$ is a binary mask. Unlike fixed-cutoff low-pass filtering, the retained set is chosen adaptively by energy ranking. The reconstruction error is exactly the discarded normalized energy,
\begin{equation}
\|x-\tilde{x}\|_2^2=\frac{1}{N}\sum_{k\notin\mathcal{K}_{\mathrm{retain}}}|X_k|^2,
\label{eq:recerr}
\end{equation}
which yields a principled rule for choosing $K_{\mathrm{freq}}$: pick the smallest value retaining $\ge95\%$ of the total energy. Moreover, the energy ranking is optimal among all selections of the same size: for any $\mathcal{K}_{\mathrm{random}}$ with $|\mathcal{K}_{\mathrm{random}}|=K_{\mathrm{freq}}$, the discarded energy of $\mathcal{K}_{\mathrm{retain}}$ is minimized.

\subsection{Synthetic Data Generation via DEFT}

To generate diverse synthetic inputs, we perturb amplitudes and phases inside the retained subspace: for the $j$-th synthetic waveform,
\begin{equation}
\tilde{X}_k^{(j)}=
\begin{cases}
a_k^{(j)}e^{\mathrm{i}\phi_k^{(j)}}, & k\in\mathcal{K}_{\mathrm{retain}},\\
0, & \text{otherwise},
\end{cases}
\label{eq:sparse}
\end{equation}
followed by a full IDFT $u^{(j)}=\mathcal{F}^{-1}[\tilde{X}^{(j)}]$, where phases $\phi_k^{(j)}\sim U[0,2\pi)$ and amplitudes $a_k^{(j)}$ match the amplitude statistics of $\mathcal{D}_{\mathrm{obs}}$. Varying these parameters yields periodic, step, and composite operating profiles. Each synthetic waveform is then passed through the physics model to obtain the labeled pair $(u^{(j)},y^{(j)})$, guaranteeing the physical consistency of the synthetic training set.

\begin{algorithm}[!htbp]
\caption{DEFT Dominant-Frequency Identification and Synthetic Data Generation}
\label{alg:deft}
\begin{algorithmic}[1]
\REQUIRE Observation set $\mathcal{D}_{\mathrm{obs}}=\{x^{(1)},\dots,x^{(M)}\}$; retained frequency count $K_{\mathrm{freq}}$; synthetic sample count $M_{\mathrm{synth}}$
\ENSURE Retained frequency set $\mathcal{K}_{\mathrm{retain}}$; training set $\mathcal{S}$
\STATE {Step 1: Spectral analysis.} For $i=1$ to $M$, compute $X_k^{(i)}$ by Eq.~\eqref{eq:dft} and $\bar{E}_k$ by Eq.~\eqref{eq:avgenergy}.
\STATE {Step 2: Dominant-frequency identification.} Sort $\bar{E}_{k(1)}\ge\cdots\ge\bar{E}_{k(N)}$; set $\mathcal{K}_{\mathrm{retain}}=\{k(1),\dots,k(K_{\mathrm{freq}})\}$ (or the smallest $K_{\mathrm{freq}}$ retaining $\ge95\%$ energy).
\STATE {Step 3: Synthetic waveform generation.} For $j=1$ to $M_{\mathrm{synth}}$: sample $a_k^{(j)},\phi_k^{(j)}$ for $k\in\mathcal{K}_{\mathrm{retain}}$, build $\tilde{X}^{(j)}$ by Eq.~\eqref{eq:sparse}, and compute $u^{(j)}$ by IDFT.
\STATE {Step 4: Physics simulation.} Feed each $u^{(j)}$ into the physics model and collect $(u^{(j)},y^{(j)})$.
\STATE {Step 5: Quality assessment.} Evaluate the trained model on a validation set; if the target is unmet, adjust $K_{\mathrm{freq}}$ or enrich waveform diversity, and iterate.
\RETURN $\mathcal{K}_{\mathrm{retain}}$, $\mathcal{S}=\{(u^{(j)},y^{(j)})\}_{j=1}^{M_{\mathrm{synth}}}$
\end{algorithmic}
\end{algorithm}

Algorithm~\ref{alg:deft} summarizes the procedure. Each DFT/IDFT incurs $O(N\log N)$ time via FFT, which dominates the total complexity of $O((M+M_{\mathrm{synth}})N\log N)$. Only the high-quality observation samples used for spectral analysis and the generated training set need storage, yielding a space complexity of $O((M+M_{\mathrm{synth}})N)$. Since $M\ll M_{\mathrm{synth}}$ (e.g., $M=10$ vs.\ $M_{\mathrm{synth}}=10^{3}$), spectral analysis of a handful of observations guides the generation of several orders of magnitude more training data. The frequency-domain compression ratio $\rho=K_{\mathrm{freq}}/N$ is typically $0.05$--$0.2$. In practice, this implies that merely $5\%$ of the frequency components suffice to reconstruct the essential signal features. Note, however, that truncation does not reduce the forward cost but rather lowers the intrinsic complexity of the synthetic data, benefiting training efficiency and generalization (Section~\ref{sec:theory}). To avoid confusion, we emphasize that the deployed model still consumes the raw $d$-dimensional time-domain input; dimensionality reduction applies to the \emph{generation} process, not the inference interface.

\subsection{Spectral-Score Filtering for Dataset Compression}
\label{sec:DEFT_filtering}

Algorithm~1 addresses the small-sample regime by extracting spectral
signatures from a handful of observations and synthesizing a large and
physically consistent training set via DEFT. However, many operator-learning
workflows face the opposite situation: large simulation
campaigns or public benchmarks such as PDEBench~\citep{takamoto2022pdebench}
already provide far more trajectories than the training budget allows, and
the limiting factor shifts from data \emph{generation} to data
\emph{selection}. A natural question is therefore: given a candidate pool
$\mathcal{D}_{\mathrm{obs}}$ and a budget $\rho\in(0,1]$, which fraction
$\rho$ of the trajectories carries the most learning value?

DEFT suggests a cheap and model-free solution. Steps~1--2 of Algorithm~1 identify the dominant frequency subspace $\mathcal{K}_{\mathrm{retain}}$ that carries $\ge\gamma$ of the energy averaged in the group, i.e., the subspace that carries the dominant variability of the initial conditions. The same subspace induces a per-sample value criterion: a trajectory whose initial condition concentrates most of its energy inside $\mathcal{K}_{\mathrm{retain}}$ lies on the dominant spectral manifold of the dataset and is intrinsically more representative and hence more informative per unit of training budget than one whose energy leaks into rarely excited modes.

\begin{equation}
\label{eq:score}
s_i \;=\;
\frac{\sum_{k\in\mathcal{K}_{\mathrm{retain}}}|X_k^{(i)}|^2}
     {\sum_{k}|X_k^{(i)}|^2},
\end{equation}

where $X_k^{(i)}$ is the DFT of the initial DC-removed condition
$u^{(i)}(\cdot,t{=}0)$ along the spatial axis. Ranking the pool by $s_i$
prefix and keeping the top-$\lfloor\rho M\rfloor$ prefix yields a nested family of
subsets: any smaller budget is a prefix of any larger one, so a single
ranking serves all compression rates.

The \textsc{Filter} branch of Algorithm~2 formalizes this procedure. Section~\ref{sec:filtering} evaluates the resulting compression behavior on two 1D PDEs of PDEBench.

\begin{algorithm}[!htbp]
\caption{DEFT: Dominant-Frequency Identification and Its Applications}
\label{alg:deft_filter}
\begin{algorithmic}[1]
\REQUIRE Training pool $\mathcal{D}_{\mathrm{obs}}=\{x^{(i)}\}_{i=1}^{M}$;
  energy threshold $\gamma=0.95$; track $\in\{\textsc{Generate},\textsc{Filter}\}$
\ENSURE Retained frequencies $\mathcal{K}_{\mathrm{retain}}$; training set $\mathcal{S}$
\STATE \textbf{Spectral analysis:} remove DC from each $x^{(i)}$, compute
  $X_k^{(i)}$ by Eq.~\eqref{eq:dft} and pool-averaged $\bar{E}_k$ by Eq.~\eqref{eq:avgenergy}.
\STATE \textbf{Frequency identification:} take
  $\mathcal{K}_{\mathrm{retain}}$ as the smallest top-$K$ set of $\bar{E}_k$
  covering $\ge\gamma$ of total energy.
\IF{track $=$ \textsc{Generate}}
  \STATE \textbf{Synthesize:} sample amplitudes/phases on
    $\mathcal{K}_{\mathrm{retain}}$ (Eq.~\eqref{eq:sparse}), invert by IDFT,
    simulate with the physics model, and iterate on $K$ if validation is unmet;
    $\mathcal{S}\leftarrow\{(u^{(j)},y^{(j)})\}$.
\ELSE
  \STATE \textbf{Filter:} score each sample by
    $s_i=\sum_{k\in\mathcal{K}_{\mathrm{retain}}}|X_k^{(i)}|^2/\sum_k|X_k^{(i)}|^2$,
    and let $\mathcal{S}$ be the top-$\lfloor\rho M\rfloor$ of
    $\mathcal{D}_{\mathrm{obs}}$ ranked by $s_i$.
\ENDIF
\RETURN $\mathcal{K}_{\mathrm{retain}}$, $\mathcal{S}$
\end{algorithmic}
\end{algorithm}

\section{Generalization Error Bound}
\label{sec:theory}

\subsection{Setup and Assumptions}

Let $\mathcal{X}\subset\mathbb{R}^d$ be a bounded input domain, $\mathcal{D}$ a data distribution on $\mathcal{X}$, $F:\mathcal{X}\to\mathbb{R}$ the target solution operator, and $N_\theta\in\mathcal{H}$ a neural operator, budget $\rho\in(0,1]$. For a frequency set $S$ with $|S|=K$, let $P_S$ be the frequency-projection operator retaining the modes in $S$, and $x_S=\mathrm{IDFT}(P_S\hat{x})$ the reconstructed signal. Let $N_S^{*}$ denote the best in-class approximation of $F$ on projected inputs. 
The truncated energy is defined as
$\displaystyle \Delta_S^2 := \mathbb{E}\|x-x_S\|^2 = \sum_{k\notin S}\mathbb{E}|X[k]|^2$
where the equality follows from Parseval's identity.

The model is trained on synthetic data $\{(x_S,F(x_S))\}$ but deployed on raw inputs $x$. Applying the triangle inequality twice yields
\begin{align}
\varepsilon :=&\ \mathbb{E}|F(x)-N_\theta(x)|\nonumber\\
\le&\ \underbrace{\mathbb{E}|F(x_S)-N_S^{*}(x_S)|}_{\text{approximation}}
+\underbrace{\mathbb{E}|N_S^{*}(x_S)-N_\theta(x_S)|}_{\text{statistical/optimization}}\nonumber\\
&+\underbrace{\mathbb{E}|F(x)-F(x_S)|}_{\text{truncation}}
+\underbrace{\mathbb{E}|N_\theta(x_S)-N_\theta(x)|}_{\text{deployment}}.
\label{eq:decomp}
\end{align}

We assume:
(A1) $F$ is $L_F$-Lipschitz on $\mathcal{X}$;
(A2) input signals lie in $H^\alpha(\mathbb{T})$, $\alpha>\tfrac12$, with $\mathbb{E}\|x\|_{H^\alpha}^2\le M_\alpha$;
(A3) $\mathcal{D}$ has bounded density $\rho_{\max}$;
(A4) realizability: $\mathcal{H}$ interpolates any finite training set;
(A5) spectral-norm control $\|W_k\|_{\mathrm{op}}\le M_k$;
(A6) the training algorithm is uniformly stable with constant $\beta=O(L_NB_X/n)$~\citep{hardt2016train}, where $L_N$ is the network Lipschitz constant, $B_X$ bounds the input norm, and $n$ is the sample size.

\subsection{Key Lemmas}

\begin{lemma}[Lipschitz properties]
\label{lem:lipschitz}
Under \emph{(A1)} and \emph{(A5)}, there exist constants $L_F,L_N>0$ such that for all $x,x'\in\mathcal{X}$,
$$
|F(x)-F(x')|\le L_F\|x-x'\|,\qquad
|N_\theta(x)-N_\theta(x')|\le L_N\|x-x'\|.
$$
Moreover, $F,N_\theta\in W^{1,\infty}(\mathcal{X})\subset W^{1,1}(\mathcal{X})$.
\end{lemma}

\begin{proof}
For a broad class of dynamical systems, the solution operator $F$ is Lipschitz continuous with respect to the input signal. For parabolic systems, this follows from analytic-semigroup estimates and Gronwall's inequality~\citep{evans2010partial}; for finite-dimensional controlled ODEs it follows from the continuous dependence of solutions on parameters and forcing terms. In either case, the output is a bounded linear or nonlinear functional of the state trajectory, and the Lipschitz constant of the input-to-output map is finite on bounded input domains. For a depth-$L$ ReLU network, each layer is $M_k$-Lipschitz by \emph{(A5)} and the ReLU activation is 1-Lipschitz, so the network is Lipschitz with constant $L_N=\prod_{k=1}^L M_k$~\citep{bartlett2017spectrally}. Bounded Lipschitz functions in a bounded domain are weakly differentiable with an essentially bounded weak gradient, hence belong to $W^{1,\infty}(\mathcal{X})$.
\end{proof}

\begin{lemma}[Approximation error]
\label{lem:approx}
Let $F\in W^{s,1}(\mathcal{X})$ with $s>0$. Under \emph{(A3)}, there exists a ReLU network $N^{*}$ of width $m$ such that
\[
\mathbb{E}_{x\sim\mathcal{D}}[|F(x)-N^{*}(x)|]\le C\cdot m^{-s/d},
\]
where $C$ depends on the Sobolev norm of $F$ and on $\rho_{\max}$.
\end{lemma}

\begin{proof}
Partition $\mathcal{X}$ into $m$ cubes of diameter $\delta\sim m^{-1/d}$ and let $P_m$ be the piecewise-constant approximation taking the mean of $F$ on each cube. By the Sobolev--Poincar\'e inequality on each cube and summation, $\|F-P_m\|_{L^p(\mathcal{X})}\le C'm^{-s/d}\|D^sF\|_{L^p(\mathcal{X})}$. Yarotsky's approximation theorem~\citep{yarotsky2017error} shows that a ReLU network of width $O(m)$ approximates $P_m$ to arbitrary precision, and the bounded-density assumption \emph{(A3)} converts the $L^1$ bound into the expected error bound.
\end{proof}

\begin{remark}
Lemma~\ref{lem:approx} is stated for the raw $d$-dimensional domain. When the model is trained on the $K$-dimensional projected domain $T_S(\mathcal{X})$, the effective dimension is $K$ rather than $d$, improving the approximation error to $O(m^{-s/K})$. This dimensionality reduction is the core theoretical benefit of IDFT sampling.
\end{remark}

\begin{lemma}[Statistical and optimization error]
\label{lem:stat}
Under \emph{(A4)}--\emph{(A6)}, with probability at least $1-\delta$ over the training sample,
\[
\mathbb{E}_{x\sim\mathcal{D}}[|N_S^{*}(x_S)-N_\theta(x_S)|]
\le \varepsilon_{\mathrm{opt}}+O\!\left(L_NB_X\sqrt{\frac{\log(1/\delta)}{n}}\right)
\]
where $\varepsilon_{\mathrm{opt}}$ is the empirical-risk suboptimality.
\end{lemma}

\begin{proof}
By \emph{(A4)}, the training labels are noiseless and the empirical risk minimizer achieves zero training error, so the excess empirical risk reduces to $\varepsilon_{\mathrm{opt}}$. Uniform stability \emph{(A6)} with constant $\beta=O(L_NB_X/n)$ bounds the sensitivity of the learned network to any single training sample. McDiarmid's inequality then yields a high-probability concentration of the population risk around the empirical risk at rate $O(\beta\sqrt{n\log(1/\delta)})=O(L_NB_X\sqrt{\log(1/\delta)/n})$; the deterministic stability term $O(\beta)=O(L_NB_X/n)$ is dominated by this concentration term.
\end{proof}

\subsection{Main Theorem}

\begin{theorem}[Generalization error upper bound]
\label{thm:main}
Under \emph{(A1)}--\emph{(A6)}, with probability at least $1-\delta$,
\begin{equation}
\varepsilon\ \le\ C_1 K^{-(\alpha-\frac12)}+C_2\, m^{-s/K}+C_3\sqrt{\tfrac{\log(1/\delta)}{n}}+\varepsilon_{\mathrm{opt}},
\label{eq:main}
\end{equation}
where $K$ is the number of retained Fourier modes, $m$ the network width, $n$ the training sample size, $s$ the Sobolev regularity of $F$, and $\varepsilon_{\mathrm{opt}}$ the optimization residual. The constants are $C_1=(L_F+L_N)C_\alpha M_\alpha^{1/2}$, $C_2$ depends on the Sobolev norm of $F$ on the projected domain and on $\rho_{\max}$, and $C_3=L_NB_X$.
\end{theorem}

\begin{proof}
We bound each term in \eqref{eq:decomp} separately.

\emph{Truncation and deployment.} By Lemma~\ref{lem:lipschitz} and Parseval's identity,
\[
\mathbb{E}|F(x)-F(x_S)|+\mathbb{E}|N_\theta(x_S)-N_\theta(x)|
\le (L_F+L_N)\Delta_S.
\]
Assumption \emph{(A2)} and the Sobolev spectral decay $|\hat{x}[k]|^2\le C_\alpha\|x\|_{H^\alpha}^2k^{-2\alpha}$ imply
\[
\Delta_S^2=\sum_{k\notin S}\mathbb{E}|X[k]|^2\le C_\alpha M_\alpha\sum_{k>K}k^{-2\alpha}
\le C_\alpha M_\alpha\frac{K^{1-2\alpha}}{2\alpha-1},
\]
where the last step uses the integral test. Hence $\Delta_S=O(K^{-(\alpha-1/2)})$ and the combined truncation--deployment error is $O((L_F+L_N)K^{-(\alpha-1/2)})$.

\emph{Approximation.} Applying Lemma~\ref{lem:approx} on the $K$-dimensional projected domain gives $O(m^{-s/K})$.

\emph{Statistical and optimization.} Lemma~\ref{lem:stat} gives $\varepsilon_{\mathrm{opt}}+O(L_NB_X\sqrt{\log(1/\delta)/n})$.

Combining the three bounds yields \eqref{eq:main}.
\end{proof}

\begin{remark}[Trade-offs and the optimal $K$]
\label{rem:tradeoff}
The terms in \eqref{eq:main} correspond to the three design choices of DEFT. The truncation term decays polynomially in $K$, faster for smoother inputs. The approximation term increases with $K$ because the effective input dimension of the projected domain grows; indeed $\frac{d}{dK}m^{-s/K}=m^{-s/K}\frac{s\ln m}{K^2}>0$, an instance of the curse of dimensionality. The statistical term decays with $n$, which synthetic data can increase on demand. Since truncation and approximation push $K$ in opposite directions, an optimal budget $K^{*}$ exists, rationalizing the empirically observed optimal range $K\in[40,60]$. The first-order condition for $K^{*}$ is obtained by setting the derivative of $C_1K^{-(\alpha-1/2)}+C_2m^{-s/K}$ to zero.
\end{remark}

\begin{corollary}[Optimality of average-energy top-$K$]
\label{cor:topk}
Fix $|S|=K$ and let $e_k:=\mathbb{E}|X[k]|^2$. The set $S_K^{\star}\in\arg\max_{|S|=K}\sum_{k\in S}e_k$ minimizes $\Delta_S$ over all size-$K$ sets, hence minimizes the truncation-and-deployment error. In particular, low-pass truncation is the special case of a monotonically decaying spectrum.
\end{corollary}

\begin{proof}
By Parseval's identity, $\Delta_S^2=\sum_{k\notin S}e_k$. Since the total energy $\sum_k e_k$ is independent of $S$, minimizing $\sum_{k\notin S}e_k$ is equivalent to maximizing $\sum_{k\in S}e_k$, which is achieved by selecting the $K$ frequencies with the largest average energy.
\end{proof}

\section{Experiments}

\subsection{PDE Operator Learning}
\label{sec:pde}

\begin{table}[!htbp]
\centering
\caption{Overall MSE$\downarrow$/$R^2$$\uparrow$ on held-out test initial
conditions. Burgers/A--C: five spatial points over the full horizon; D--S:
full grid at four horizon fractions, nine-seed mean.}
\label{tab:pde}
\scriptsize
\setlength{\tabcolsep}{4pt}
\begin{tabular}{@{}lccc@{}}
\toprule
Method & Burgers & Allen--Cahn & 1D diff--sorp\\
\midrule
DEFT (ours) & $\mathbf{1.4\times10^{-1}}$\,/\,$\mathbf{0.7309}$ & $\mathbf{3.9\times10^{-2}}$\,/\,$\mathbf{0.9507}$ & $\mathbf{1.3\times10^{-4}}$\,/\,$\mathbf{0.9978}$ \\
Random & $4.5\times10^{-1}$\,/\,0.1411 & $4.8\times10^{-2}$\,/\,0.9387 & $3.2\times10^{-4}$\,/\,0.9946 \\
LHS & $4.9\times10^{-1}$\,/\,0.0584 & $8.5\times10^{-2}$\,/\,0.8926 & $1.6\times10^{-4}$\,/\,0.9973 \\
Bayesian & $5.0\times10^{-1}$\,/\,0.0472 & $4.7\times10^{-2}$\,/\,0.9411 & $3.3\times10^{-4}$\,/\,0.9944 \\
\bottomrule
\end{tabular}
\end{table}

Table~\ref{tab:pde} compares the sampling methods on three PDE problems.
DEFT is clearly best on the low-frequency-dominated Burgers
($R^2=0.7309$ vs.\ $\le0.1411$ for all baselines) and Allen--Cahn
($0.9507$), where shocks and phase fronts concentrate energy at low modes.
Diffusion--sorption is instead boundary-driven (only $K=2$ retained
modes); DEFT still attains the best mean accuracy ($R^2=0.9978$; nine seed
combinations), significantly beating Random and Bayesian (paired Wilcoxon
$p=0.002$) and matching LHS with lower run-to-run
variance.

\subsection{PDEBench Data Value Filtering}
\label{sec:filtering}

\begin{figure}[!htbp]
\centering
\begin{subfigure}{0.48\columnwidth}
\centering
\includegraphics[width=\linewidth]{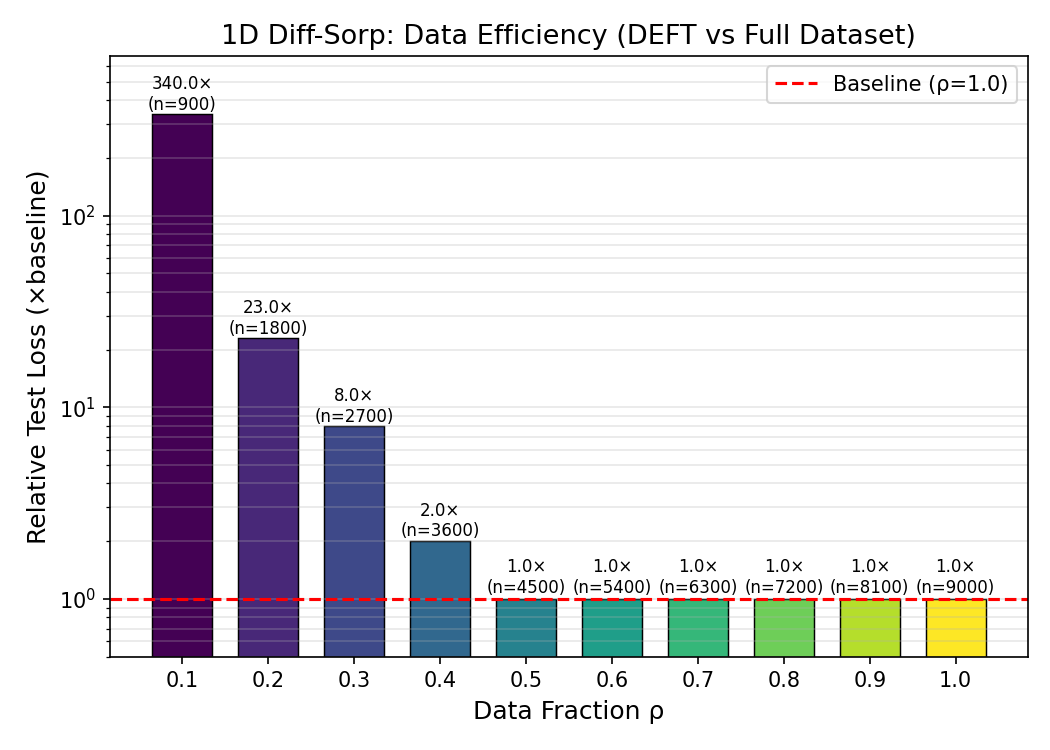}
\caption{1D diffusion--sorption.}
\label{fig:filtering}
\end{subfigure}\hfill
\begin{subfigure}{0.48\columnwidth}
\centering
\includegraphics[width=\linewidth]{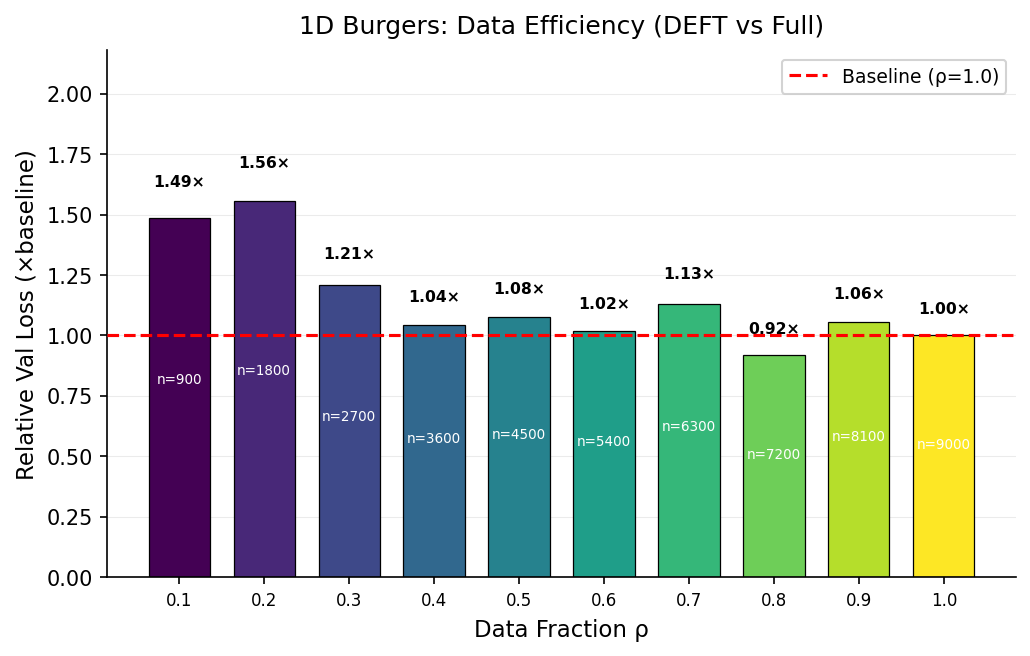}
\caption{1D Burgers.}
\label{fig:filtering2}
\end{subfigure}
\caption{Data-value filtering on two PDEBench 1D equations.}
\label{fig:filtering_all}
\end{figure}

The DEFT pipeline replaces subset selection with data generation, but the same spectral score can also rank samples that already exist, complementing data generation with data selection. We evaluate this filtering strategy on two distinct benchmarks from PDEBench~\citep{takamoto2022pdebench}: the 1D diffusion--sorption equation, whose smooth initial conditions are strongly low-frequency dominated, and the 1D Burgers equation, whose developing shock structures introduce sharper gradients and richer spectral content. For each dataset, we compute the globally dominant spectral subspace (the top-frequency set covering $95\%$ of the training-averaged power spectrum) and score every training sample by the fraction of its initial-condition energy captured inside this subspace. Retaining the top fraction $\rho$ by this spectral score yields a principled compression of the training-set. In diffusion --sorption, the top-50\% subset already matches the loss of the full-data test ($1.0\times10^{-6}$); in Burgers, the top-40\% subset stays within $5\%$ of full-data performance and the top-80\% subset slightly exceeds it.

\FloatBarrier
\subsection{Battery Remaining-Capacity Prediction}
\label{sec:battery}

We evaluate DEFT on battery remaining-capacity prediction using the COMSOL P2D--SEI model as the physics-based model. Training data are single-tone cosine current profiles $N(t)=A\cos(2\pi w t)$ with $A\in\{0.1,\dots,1.0\}$. The physics solver returns the capacity trajectory for each profile. The attention-enhanced DeepONet maps the current history to remaining capacity. We test on four synthetic waveforms of increasing complexity, the industrial CC-CV protocol and three real EV driving profiles.

Across the waveform difficulty gradient (Figure~\ref{fig:synthetic}), accuracy degrades only mildly with complexity: the square wave attains $R^2$ up to $0.998$, and even the most complex double-V waveform stays at $R^2\ge0.976$; within each waveform, the faster current variation decreases $R^2$. Errors cluster at abrupt current changes, where high-frequency content concentrates; truncating the square wave's harmonics (which decay as $1/n$) produces the expected Gibbs overshoot of about $9\%$ near switching instants, yet the retained low-frequency components still carry the dominant battery dynamics, consistent with the $O(K^{-(\alpha-1/2)})$ truncation term of Theorem~\ref{thm:main}.

\begin{figure}[H]
\centering
\includegraphics[width=0.65\columnwidth]{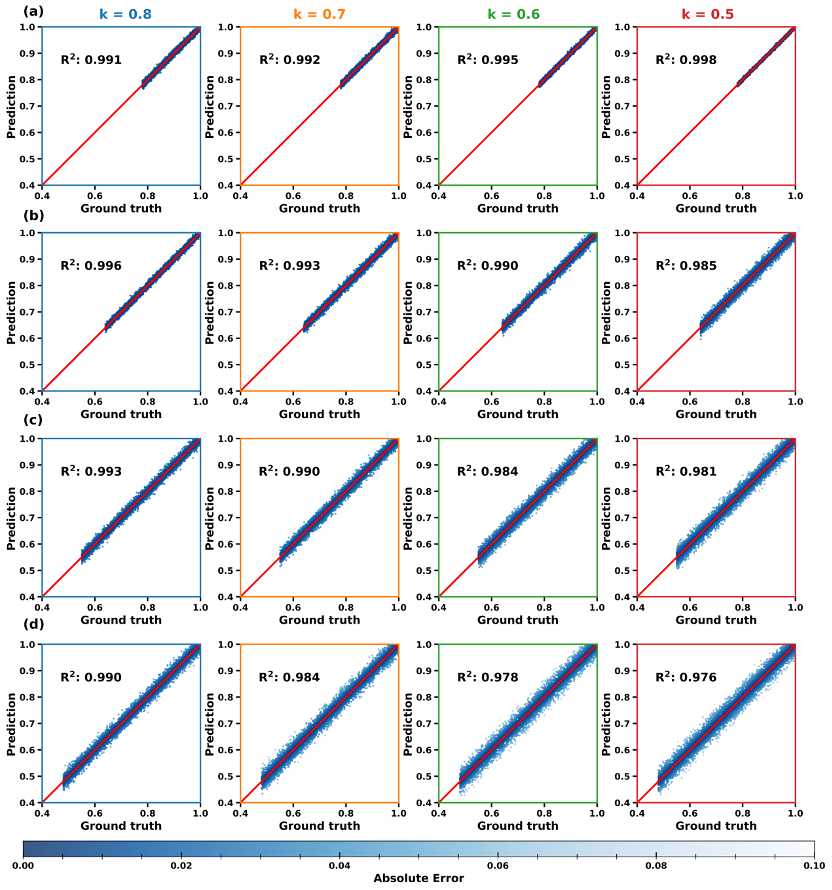}
\caption{Predicted vs.\ simulated remaining capacity for four synthetic waveforms (rows: square, V-shaped, offset-V, double-V) at four current variation rates (columns, $k=0.8$ to $0.5$). Color encodes absolute error; errors concentrate where the current changes abruptly.}
\label{fig:synthetic}
\end{figure}

On the CC-CV protocol the model achieves $R^2=0.827$; error levels in the constant-current and constant-voltage phases are comparable and there is no performance jump at the phase transition. On real EV profiles (Table~\ref{tab:real}), prediction MSE stays near $10^{-5}$; EV2 attains $R^2=0.980$, \emph{higher} than the smoother EV1 ($0.903$), while the highly dynamic EV3 still maintains $0.845$.

\begin{table}[!htbp]
\centering
\caption{Generalization under real operating conditions.}
\label{tab:real}
\small
\begin{tabular}{@{}lcc@{}}
\toprule
Operating condition & MSE & $R^2$ \\
\midrule
CC-CV charging protocol & $1.18\times10^{-5}$ & 0.827 \\
EV1 (smooth urban) & $1.10\times10^{-5}$ & 0.903 \\
EV2 (mixed suburban) & $1.06\times10^{-5}$ & {0.980} \\
EV3 (dynamic highway) & $1.00\times10^{-5}$ & 0.845 \\
\bottomrule
\end{tabular}
\end{table}

Sweeping the number of retained base frequencies $K$ from 10 to 80 reveals a staged pattern (Figure~\ref{fig:ksweep}): rises quickly up untill $K\approx60$ and saturates or slightly degrades beyond. At $K=70$, for instance, the mean $R^2$ dips below that at $K=60$, consistent with noise amplification through spurious high-frequency oscillations and increased approximation difficulty, exactly the trade-off predicted by Remark~\ref{rem:tradeoff}. 

\begin{figure}[H]
\centering
\includegraphics[width=0.5\columnwidth]{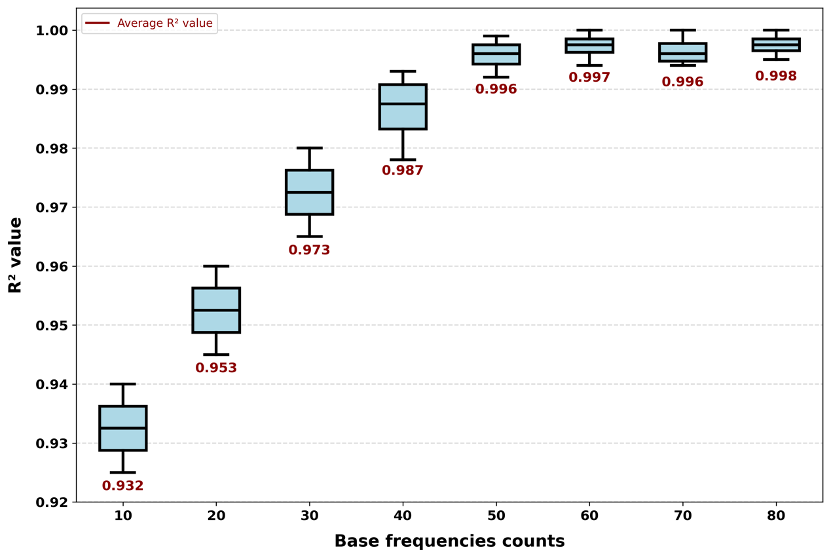}
\caption{Effect of the retained base-frequency count $K$.}
\label{fig:ksweep}
\end{figure}

The base model is trained on lithium-ion cells with LiPF$_6$ in EC:EMC $3{:}7$ electrolyte, a graphite anode, and an NCA cathode. For new chemistries we freeze the branch/trunk core and append two output layers, which are fine-tuned on a small fraction of target-domain data (Figure~\ref{fig:transfer}). With 20\% fine-tuning data, $R^2$ improves from $0.712$ to $0.989$ on EC:EMC $1{:}1$ lithium-ion cells, from $0.735$ to $0.993$ on silicon-anode cells, from $0.720$ to $0.985$ on NMC622-cathode cells, from $0.700$ to $0.971$ on sodium-ion cells, and from $0.710$ to $0.965$ on liquid-hydrogen cells. Chemically similar systems exceed $R^2=0.9$ with only 5--10\% fine-tuning data, whereas the mechanistically distinct liquid-hydrogen system needs 15--20\%. This transferability likely reflects the shared \emph{operator structure} across chemistries -- a mapping from a time-varying current to spatiotemporal capacity~\citep{severson2019data,attia2020closed}.

\begin{figure}[!htbp]
\centering
\includegraphics[width=0.85\columnwidth]{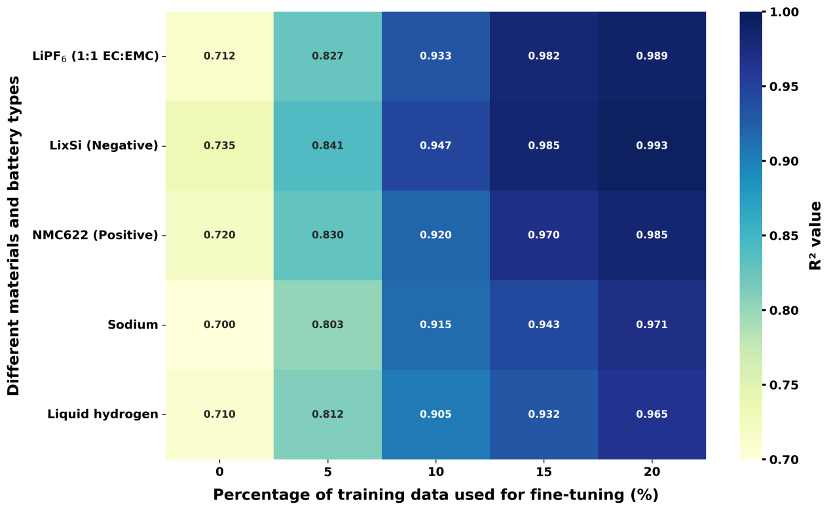}
\caption{Transfer to new battery chemistries. $R^2$ after fine-tuning two appended output layers on 0--20\% of target-domain data; the backbone (branch/trunk feature extractors) is frozen.}
\label{fig:transfer}
\end{figure}

\section{Discussion}

This study introduces DEFT, a principled data sampling optimazition method that shifts data acquisition from passive time-domain collection to active frequency-domain construction.

Several limitations remain, such as the regularity assumptions underlying our bound. Theoretically, the bound could be extended to weaker regularity classes and sharper decay laws.  Methodologically, the spectral criterion could becombined with multi-scale bases to handle broadband systems. Addressing these issues will further broaden the applicability of DEFT, and we believe DEFT sampling offers an effective and principled route toward efficient data sampling.

\bibliographystyle{unsrtnat}
\bibliography{refs}

\end{document}